\documentclass[runningheads]{llncs}
\usepackage{graphicx}
\graphicspath{{Figures/}}
\usepackage{amssymb}
\usepackage{amsmath}
\usepackage{multirow}
\usepackage{arydshln}
\usepackage[switch]{lineno} 

\newcommand{\comment}[1]{}
\newtheorem{prop}{Proposition}

\usepackage{collcell}
\usepackage{booktabs}
\usepackage{siunitx}
\usepackage[table]{xcolor}
\usepackage{tikz}

\newcolumntype{P}[3]{
  >{\collectcell{\colorfromval{#1}{#2}{#3}}} l <{\endcollectcell}}
  
\newcommand*{\colorfromval}[4]{%
  \pgfmathparse{(#4-#1)/(#2-#1)*100}%
  \global\let\colorratio\pgfmathresult
  
  \cellcolor{orange!\colorratio!white!50}%
  \tablenum[#3]{#4}}

\begin{document}
\title{Information Interaction Profile of Choice Adoption}
\author{Gael Poux-Medard\inst{1}\orcidID{0000-0002-0103-8778} \and
Julien Velcin\inst{1}\orcidID{0000-0002-2262-045X} \and
Sabine Loudcher\inst{1}\orcidID{0000-0002-0494-0169}}

\authorrunning{G. Poux-Médard et al.}
\institute{$^1$ Université de Lyon, ERIC EA 3083, France\\
\email{gael.poux-medard@univ-lyon2.fr}\\
\email{julien.velcin@univ-lyon2.fr}\\
\email{sabine.loudcher@univ-lyon2.fr}\\
}
\maketitle              % typeset the header of the contribution

\begin{abstract}
Interactions between pieces of information (\textit{entities}) play a substantial role in the way an individual acts on them: adoption of a product, the spread of news, strategy choice, etc. However, the underlying interaction mechanisms are often unknown and have been little explored in the literature.
We introduce an efficient method to infer both the entities interaction network and its evolution according to the temporal distance separating interacting entities; together, they form the \textit{interaction profile}. The interaction profile allows characterizing the mechanisms of the interaction processes.
We approach this problem \textit{via} a convex model based on recent advances in multi-kernel inference. We consider an ordered sequence of exposures to entities (URL, ads, situations) and the actions the user exerts on them (share, click, decision). We study how users exhibit different behaviors according to \textit{combinations} of exposures they have been exposed to. We show that the effect of a combination of exposures on a user is more than the sum of each exposure's independent effect--there is an interaction. We reduce this modeling to a non-parametric convex optimization problem that can be solved in parallel. 
Our method recovers state-of-the-art results on interaction processes on three real-world datasets and outperforms baselines in the inference of the underlying data generation mechanisms. Finally, we show that interaction profiles can be visualized intuitively, easing the interpretation of the model.
\end{abstract}

\section{Introduction}
When told in the year 2000 that the XX$^{th}$ century was the century of physics and asked whether he agrees that the next one would be the century of biology, Stephen Hawkins answered that he believed the XXI$^{th}$ century would be the century of complexity. Be it a reasoned forecast or a tackle to promote scientific multidisciplinarity, there has been an undeniable growing interest for complex systems in research over the past decades. A complex system can be defined as a system composed of many components that interact with each other. Their study often involves network theory, a branch of mathematics that aims at modeling those interactions --that can be physical, biological, social, etc. A significant point of interest is understanding how information spreads along the edges of a network--with a particular interest in social networks. If the social network skeleton (edges, nodes) plays a significant role in such processes, recent studies pointed out that the interaction between spreading entities might also play a non-trivial role in it \cite{CorrelatedCascade,InteractingViruses}. The histograms presented Fig.\ref{fig:FigDistrib} illustrate this finding: the probability for a piece of information to be adopted (or spread) varies according to the exposure to another one at a previous time. We refer to this figure as the \textit{interaction profile}. The study of this quantity is a novel perspective: the interaction between pieces of information has been little explored in the literature, and no previous work aims at unveiling trends in the information interaction mechanisms.
\begin{figure}[h]
    \centering
    \includegraphics[width=1.\columnwidth]{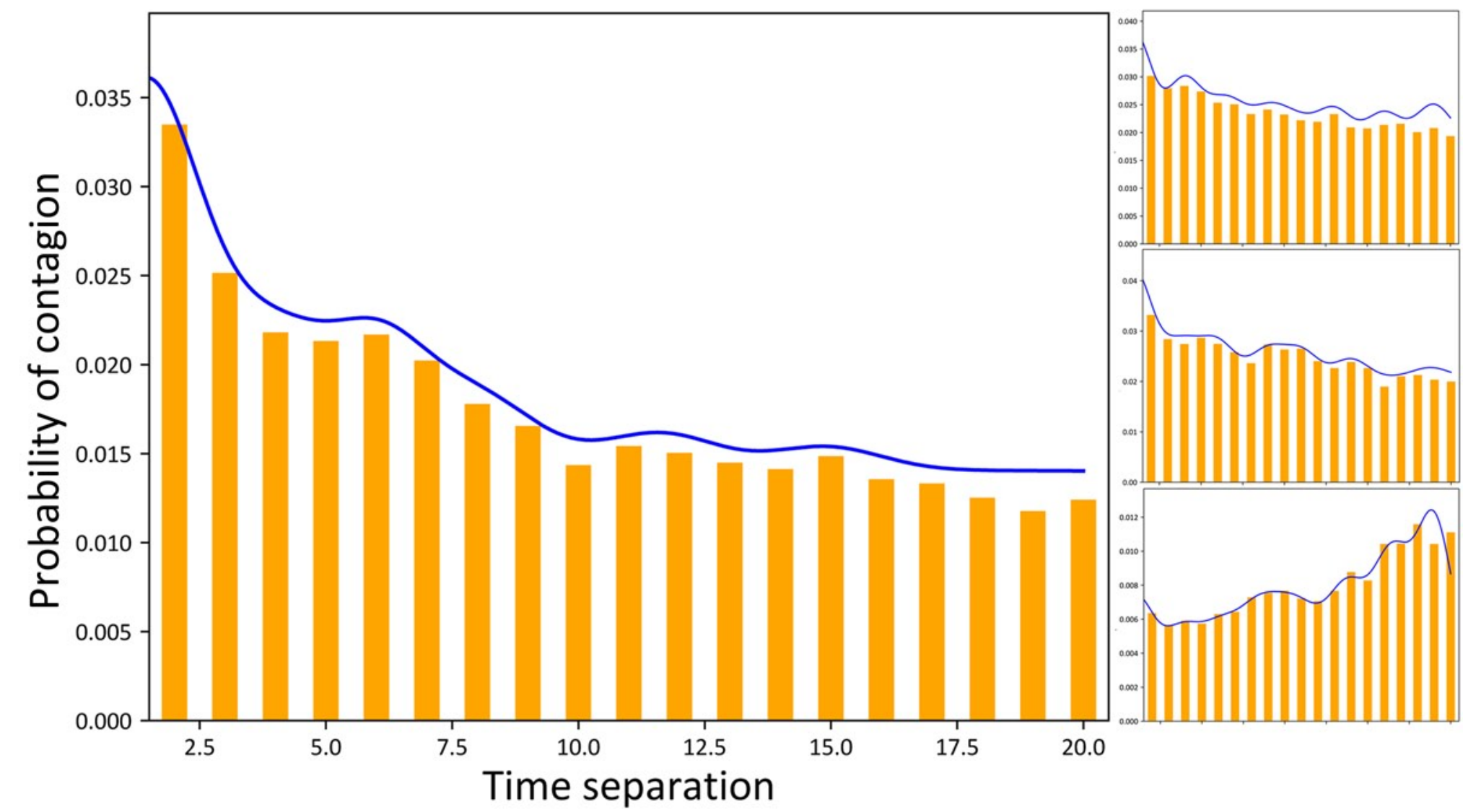}
    \caption{\textbf{Interaction profiles between pairs of entities} --- Examples of interaction profiles on Twitter; here is shown the effect of URL shortening services migre.me (left), bit.ly (right-top), tinyurl (right-middle) and t.co (right-bottom) on the probability of tweeting a t.co URL and its evolution in time. This interaction profile shows, for instance, that there is an increased probability of retweet for a t.co URL when it appears shortly after a migre.me one (interaction). This increase fades when the time separation grows (no more interaction). In blue, the interaction profile inferred by our model.}
    \label{fig:FigDistrib}
\end{figure}

The study of interactions between pieces of information (or entities) has several applications in real-world systems. We can mention the fields of recommender systems (the probability of adoption is influenced by what a user saw shortly prior to it), news propagation and control (when to expose users to an entity in order to maximize its spreading probability \cite{SpreadNewsOnlineScience}), advertising (same reasons as before \cite{AdsDataset}), choice behavior (what influences one's choice and how \cite{SergioSocialDilemma}). 

In the present work, we propose to go one step further and to unveil the mechanisms at stake within those interacting processes by inferring information interaction profiles. Let us imagine, for instance, that an internet user is exposed to an ad at time $t_1$ and to another ad for a similar product at time $t_2>t_1$. Due to the semantic similarity between the two exposures, we suppose that the exposure to the first one influences the user's sensitivity (likeliness of a click) to the second one a time $t_2-t_1$ later. Modeling this process involves quantifying the influence an ad exerts on the other and how it varies with the time separation between the exposures. The reunion of those quantities form what we call the \textit{interaction profile} --illustrated Fig.\ref{fig:FigDistrib}), that is, the influence an exposure exerts on the adoption (click, buy, choice, etc.) of another one over time.

Following this idea, we introduce an efficient method to infer both the entities (ad, tweets, products, etc.) interaction network and its evolution according to the temporal distance separating the interacting entities (the influence of an entity A on entity B will not be the same depending on whether A appeared 10 minutes or 10 hours before B). Together they form the interaction profile. 

First, we develop a method for inferring this interaction profile in a continuous-time setup using multi-kernel inference methods \cite{KernelCascade}. 
Then we show that the inference of the parameters boils down to a convex optimization problem for kernel families obeying specific properties. Moreover, the problem can be subdivided into as many subproblems as entities, which can be solved in parallel. The convexity of the problem guarantees convergence to the likelihood's global optimum for each subproblem and, therefore, to the problem's optimal likelihood. 
We apply the model to investigate the role of interaction profiles on synthetic data and in various corpora from different fields of research: advertisement (the exposure to an ad influences the adoption of other ads \cite{AdsDataset}), social dilemmas (the previous actions of one influences another's actions \cite{SergioSocialDilemma}) and information spread on Twitter (the last tweets read influence what a user retweets \cite{CoC})\footnote{Implementation codes and datasets can be found at https://anonymous.4open.science/r/f484ee48-0b47-468d-a488-6f56c85face8/}.
Finally, we provide analysis leads and show that our method recovers state-of-the-art results on interaction processes on each of the three datasets considered.

\subsection{Contributions}
The main contributions of this paper are the following:
\begin{itemize}
    \item We introduce the interaction profile, which is the combination of both the interaction network between entities and its evolution with the interaction time distance, according to the inferred kernels. The interaction profile is a powerful tool to understand how interactions take place in a given corpus (see Fig.\ref{fig:FigAnalInter}) and has not been developed in the literature. Its introduction in research is the main contribution of the present work.
    
    \item We design a convex non-parametric algorithm that can be solved in parallel, baptized InterRate. InterRate automatically infers the kernels that account the best for information interactions in time within a given kernel family. Its output is the aforementioned interaction profile.
    
    \item We show that InterRate yields better results than non-interacting or non-temporal baseline models on several real-world datasets. Furthermore, our model can recover several conclusions about the datasets from state-of-the-art works.
\end{itemize}

\section{Related work}
Previous efforts in investigating the role of interactions in information diffusion have shown their importance in the underlying spreading processes. Several works study the interaction of information with users' attention \cite{Weng2012}, closely linked to information overload concepts \cite{rodriguez2014}, but not the interaction between the pieces of information themselves. On the other hand, whereas most of the modeling of spreading processes are based on either no competition \cite{InfluenzaSpread,pouxmdard2019influential} or perfect competition \cite{WinnerTakesAll} assumption, it has been shown that relaxing this hypothesis leads to a better description of competitive spread \cite{InteractingViruses} --with the example of Firefox and Chrome web browsers, whose respective popularities are correlated. According to this finding, a significant effort has been done in elaborating complex processes to \textit{simulate} interaction \cite{Zhu2020,WinnerTakesAll} on real-world networks.

However, fewer works have been developed to tackle interaction in information spread from a machine learning point of view. The correlated cascade model \cite{CorrelatedCascade} aims to infer an interacting spreading process's latent diffusion network. In this work, the interaction is modeled by a hyper-parameter $\beta$ tuning the intensity of interactions according to an exponentially decaying kernel. In their conclusion, the authors formulate the open problem of learning several kernels and the interaction intensity parameter $\beta$, which we address in the present work.

To our knowledge, the attempt the closest to our task to model the interaction intensity parameter $\beta$ is Clash of the contagions \cite{CoC}; this aims to predict retweets on Twitter based on tweets seen by a user. This model estimates the probability of retweet for a piece of information, given the last tweets a user has been exposed to, according to their relative position in the Twitter feed. The method suffers various flaws (scalability, non-convexity). It also defines interactions based on an arguable hypothesis made on the prior probability of a retweet (in the absence of interactions) that makes its conclusions about interactions sloppy. It is worth noting that in \cite{CoC}, the authors outline the problem of the inference of the interaction profile but do so without searching for global trends such as the one shown in Fig.\ref{fig:FigDistrib}. 
Recent works address the various flaws observed in \cite{CoC} and suggest a more general approach to the estimation of the interaction intensity parameters \cite{IMMSBM}. The latter model develops a scalable algorithm that correctly accounts for interacting processes but neglects the interactions' temporal aspect. To take back the Twitter case study, it implies that in the case of a retweet at time $t$, a tweet appearing at $t_1 \ll t$ in the news feed has the same influence on the retweet as a tweet that appeared at $t_2 \approx t$. A way to relax this assumption is to integrate a temporal setting in the interaction network inference problem. 

In recent years, temporal networks inference has been a subject of interest. Significant advances have been made using survival theory modeling applied to partial observations of independent cascades of contagions \cite{NetRate,InfoPath}. 
In this context, an infected node tries to contaminate every other node at a rate that is tuned by $\beta$. 
While this work is not directly linked to ours, it has been a strong influence on the interaction profile inference problem we develop here; the problems are different, but the methodology they introduce deserves to be mentioned for its help in the building of our model (development and convexity of the problem, analogy between interaction profile and hazard rate). Moreover, advances in network inference based on the same works propose a multi-kernel network inference method that we adapted to the problem we tackle here \cite{KernelCascade}. Inspired by these works, we develop a flexible approach that allows for the inference of the best interaction profile from several candidate kernels.

\section{InterRate}
\subsection{Problem definition}
We illustrate the process to model in Figure \ref{fig:FigProc}. It runs as follows: a user is first exposed to a piece of information at time $t_0$. The user then chooses whether to act on it at time $t_0+t_{s}$ (an act can be a retweet, a buy, a booking, etc.); $t_s$ can be interpreted as the ``reaction time'' of the user to the exposure, assumed constant. The user is then exposed to the next piece of information a time $\delta t$ later, at $t_1=t_0+\delta t$ and decides whether to act on it a time $t_s$ later, at $t_1+t_{s}$, and so on. Here, $\delta t$ is the time separating two consecutive exposures, and $t_{s}$ is the reaction time, separating the exposure from the possible contagion. In the remaining of the paper, we refer to the user's action on an exposure (tweet appearing in the feed, exposure to an ad, etc.) as a contagion (retweet or the tweet, click on an ad, etc.).

\begin{figure*} % Doit être sur la même page que section InterRate
    \centering
    \includegraphics[width=1.\textwidth]{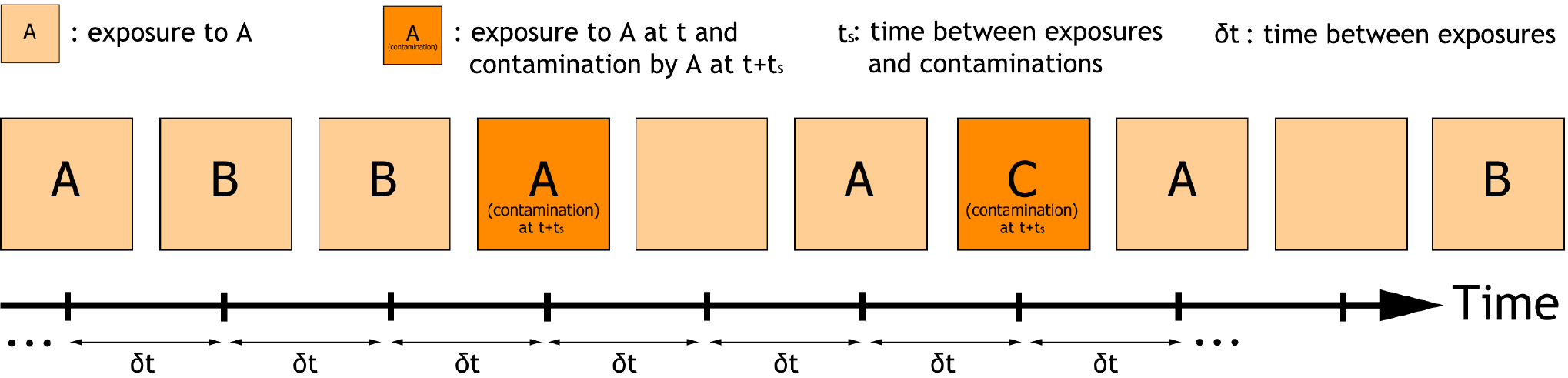}
    \caption{\textbf{Illustration of the interacting process} --- Light orange squares represent the exposures, dark orange squares represent the exposures that are followed by contagions and empty squares represent the exposures to the information we do not consider in the datasets (they only play a role in the distance between exposures when one considers the order of appearance as a time feature). A contagion occurs at a time $t_s$ after the corresponding exposure. Each new exposure arrives at a time $\delta$t after the previous one. Contagion takes place with a probability conditioned by all previous exposures. In the example, the contagion by A at time $t+t_s$ depends on the effect of the exposure to A at times $t$ and $t-3\delta t$, and to B at times $t-\delta t$ and $t-2\delta t$.}
    \label{fig:FigProc}
\end{figure*}

This choice of modeling comes with several hypotheses. 
\textbf{First}, the pieces of information a user is exposed to appear independently from each other. It is the main difference between our work and survival analysis literature: the pseudo-survival of an entity is conditioned by the random arrival of pieces of information. Therefore, it cannot be modeled as a point process. This assumption holds in our experiments on real-world datasets, where users have no influence on what information they are exposed to.
\textbf{Second} hypothesis, the user is contaminated solely on the basis of the previous exposures in the feed \cite{CoC,CorrelatedCascade}.
\textbf{Third}, the reaction time separating the exposure to a piece of information from its possible contagion, $t_s$, is constant (i.e. the time between a read and a retweet in the case of Twitter). Importantly, this hypothesis is a deliberate simplification of the model for clarity purposes; relaxing this hypothesis is straightforward by extending the kernel family, which preserves convexity and time complexity. Note that this simplification does not always hold, as shown in recent works concluding that response time can have complex time-dependent mechanisms \cite{Yu2017}.

\subsection{Likelihood}
We define the likelihood of the model of the process described in Figure \ref{fig:FigProc}. Let $t_i^{(x)}$ be the exposure to $x$ at time $t_i$, and $t_i^{(x)} + t_s$ the time of its possible contagion.
Consider now the instantaneous probability of contagion (\textit{hazard function}) $H(t_i^{(x)}+t_s \vert t_j^{(y)},\beta_{xy})$, that is the probability that a user exposed to the piece of information $x$ at time $t_i$ is contaminated by $x$ at $t_i+t_s$ given an exposure to $y$ at time $t_j \leq t_i$.
%, according to the interaction intensity parameter $\beta_{xy}$. 
The matrix of parameters $\beta_{ij}$ is what the model aims to infer. $\beta_{ij}$ is used to characterize the interaction profile between entities.
We define the set of exposures preceding the exposure to $x$ at time $t_i$ (or history of $t_i^{(x)}$) as $\mathcal{H}_i^{(x)} \equiv \{ t_j^{(y)} \leq t_i^{(x)} \}_{j,y}$. 
Let $\mathcal{D}$ be the whole dataset such as $\mathcal{D} \equiv \{ ( \mathcal{H}_i^{(x)}, t_i^{(x)}, c_{t_i}^{(x)} ) \}_{i,x}$. Here, c is a binary variable that account for the contagion ($c_{t_i}^{(x)}=1$) or non-contagion ($c_{t_i}^{(x)}=0$) of $x$ at time $t_i+t_s$. The likelihood for one exposure in the sequence given $t_j^{(y)}$ is:

{\small
    \begin{equation*}
        \begin{split}
            &L(\mathbf{\beta}_{xy} \vert \mathcal{D}, t_s) = P(\mathcal{D} \vert \mathbf{\beta}_{xy}, t_s) =\\
            &\underbrace{H(t_i^{(x)}+t_s \vert t_j^{(y)},\beta_{xy})^{c_{t_i}^{(x)}}  }_{\text{contagion at $t_i^{(x)}+t_s$ due to $t_j^{(y)}$}}\cdot
            \underbrace{(1 - H(t_i^{(x)}+t_s \vert t_j^{(y)},\beta_{xy}))^{(1-c_{t_i}^{(x)})}}_{\text{Survival at $t_i^{(x)}+t_s$ due to $t_j^{(y)}$}}
        \end{split}
    \end{equation*}
}
The likelihood of a sequence (as defined Fig.\ref{fig:FigProc}) is then the product of the previous expression over all the exposures that happened before the contagion event $t_i^{(x)} + t_s$ e.g. for all $t_j^{(y)} \in \mathcal{H}_i^{(x)}$.
Finally, the likelihood of the whole dataset $\mathcal{D}$ is the product of $L(\mathbf{\beta}_{x} \vert \mathcal{D}, t_s)$ over all the observed exposures $t_i^{(x)}$. Taking the logarithm of the resulting likelihood, we get the final log-likelihood to maximize:
\begin{equation}
    \label{Eq:likelihood}
    \begin{split}
        \ell (\mathbf{\beta} \vert \mathcal{D},t_s) =\ \ \ \ \ \ \ \ \ \ \ & \\
        \sum_{\mathcal{D}} \ \sum_{t_j^{(y)} \in \mathcal{H}_i^{(x)}}&
        c_{t_i}^{(x)} \log \left( H(t_i^{(x)}+t_s \vert t_j^{(y)},\beta_{xy}) \right)
        \\
        +\ (1-c_{t_j}^{(y)})&\log \left( 1-H(t_i^{(x)}+t_s \vert t_j^{(y)},\beta_{xy}) \right)
    \end{split}
\end{equation}

\subsection{Proof of convexity}
The convexity of a problem guarantees to retrieve its optimal solution and allows using dedicated fast optimization algorithms.
\begin{prop}
The inference problem $\min_{\beta} -\ell (\mathbf{\beta} \vert \mathcal{D},t_s) \  \forall \beta \geq 0$, is convex in all of the entries of $\beta$ for any hazard function that obeys the following conditions: 
\begin{equation}
    \label{Eq:ConvexCond}
    \begin{cases}
        &H'^2 \geq H''H  \\
        &H'^2 \geq -H''(1-H)\\
        &H \in \left]0;1\right[
    \end{cases}
\end{equation}
where $'$ and $''$ denote the first and second derivative with respect to $\beta$, and H is the shorthand notation for $H(t_i^{(x)}+t_s \vert t_j^{(y)},\beta_{xy}) \ \forall i,j,x,y$.
\end{prop}

\begin{proof}
The negative log-likelihood as defined in Eq.\ref{Eq:likelihood} is a summation of $-\log H$ and $-\log(1-H)$; therefore $H \in \left]0;1\right[$. The second derivative of these expressions according to any entry $\beta_{mn}$ (noted $''$) reads:
\begin{equation}
    \label{Eq:DemoConvex}
    \begin{cases}
    &\left( -\log H \right)'' = \left( \frac{-H'}{H} \right)' = \frac{H'^2-H''H}{H^2} \\
    &\left( -\log(1-H) \right)'' = \left( \frac{H'}{1-H} \right)' = \frac{H'^2 + H''(1-H)}{(1-H)^2}
    \end{cases}
\end{equation}
The convexity according to a single variable holds when the second derivative positive, which leads to Eq.\ref{Eq:ConvexCond}. The convexity of the problem then follows from composition rules of convexity.
\qed
\end{proof}

A number of functions obey the conditions of Eq.\ref{Eq:ConvexCond}, such as the exponential ($e^{-\beta t}$), Rayleigh ($e^{-\frac{\beta}{2} t^2}$), power-law ($e^{-\beta \log t}$) functions, and any log-linear combination of those \cite{KernelCascade}. These functions are standard in survival theory literature \cite{GRSurvivalAnalysis}. 

The final convex problem can then be written $\min_{\beta \geq 0} -\ell (\mathbf{\beta} \vert \mathcal{D},t_s)$. An interesting feature of the proposed method is that the problem can be subdivided into N convex subproblems that can be solved independently (one for each piece of information). To solve the subproblem of the piece of information $x$, that is to find the vector $\beta_x$, one needs to consider only the subset of $\mathcal{D}$ where $x$ appears. Explicitly, each subproblem consists in maximizing Eq.\ref{Eq:likelihood} over the set of observations $\mathcal{D}^{(x)} \equiv \{ ( \mathcal{H}_i^{(x)}, t_i^{(x)}, c_{t_i}^{(x)} ) \}_i$.

\section{Experimental setup}
\subsection{Kernel choice}
\subsubsection{Gaussian RBF kernel family (IR-RBF)}
Based on \cite{KernelCascade}, we consider a log-linear combination of Gaussian radial basis function (RBF) kernels as hazard function. We also consider the time-independent kernel needed to infer the base probability of contagion discussed in the section ``Background noise in the data'' below. The resulting hazard function is then:
{\small
    \begin{equation*}
    \label{Eq:H}
    \log H(t_i^{(x)}+t_s \vert t_j^{(y)},\beta_{ij}) = -\beta_{ij}^{(bg)} - \sum_{s=0}^S \frac{\beta^{(s)}_{ij}}{2} (t_i+t_s-t_j-s)^2
    \end{equation*}
}
The parameters $\beta^{(s)}$ of Rayleigh kernels are the amplitude of a Gaussian distribution centered on time $s$. The parameter S represents the maximum time shift we consider. In our setup, we set S=20.
We think it is reasonable to assume that an exposition does not significantly affect a possible contagion 20 steps later. The parameter $\beta_{ij}^{(bg)}$ corresponds to the time-independent kernel --base probability of contagion by i. The formulation allows the model to infer complex distributions from a reduced set of parameters whose interpretation is straightforward. 

\subsubsection{Exponentially decaying kernel (IR-EXP)}
We also consider an exponentially decaying kernel that can be interpreted as a modified version of a multivariate Hawkes point process (or self-exciting point process). A multivariate Hawkes process is a point process where different classes of objects exert an influence on the point process of the others \cite{PaperHawkes}. The most common (and historical) form of such a process considers a time-decaying exponential function to model a class's influence on others. However, in our setup, we cannot consider a Hawkes process: we infer the variation of contagion probabilities in time conditionally on the earlier random exposure to a piece of information. The studied process is therefore not rigorously self-exciting. We consider the following form for the hazard function and refer to this modeling as IR-EXP instead of a Hawkes process:
{\small
\begin{equation*}
        \label{Eq:HHawkes}
        \log H(t_i^{(x)}+t_s \vert t_j^{(y)},\beta_{xy}) = -\beta_{ij}^{(bg)} -\beta_{ij} (t_i+t_s-t_j)
    \end{equation*}
}
Where $\beta_{ij}^{(bg)}$ once again accounts for the background noise in the data discussed further in this section.

\subsection{Parameters learning}
Datasets are made of sequences of exposures and contagions, as shown in Fig.\ref{fig:FigProc}. 
To assess the robustness of the proposed model, we apply a 5-folds cross-validation method. After shuffling the dataset, we use 80\% of the sequences as a training set and the 20\% left as a test set. We repeat this slicing five times, taking care that an interval cannot be part of the test set more than once. The optimization is made in parallel for each piece of information via the convex optimization module for Python CVXPY.

We also set the time separating two exposures $\delta t$ as constant. It means that we consider only the order of arrival of exposures instead of their absolute arrival time. The hypothesis that the order of exposures matters more than the absolute exposure times has already been used with success in the literature \cite{CoC}. Besides, in some situations, the exact exposure time cannot be collected, while the exposures' order is known.
For instance, in a Twitter corpus, we only know in what order a user read her feed, unlike the exact time she read each of the posts. However, from its definition, our model works the same with non-integer and non-constant $\delta t$ in datasets where absolute time matters more than the order of appearance.

\subsection{Background noise in the data}
\begin{figure}
    \centering
    \includegraphics[width=0.8\textwidth]{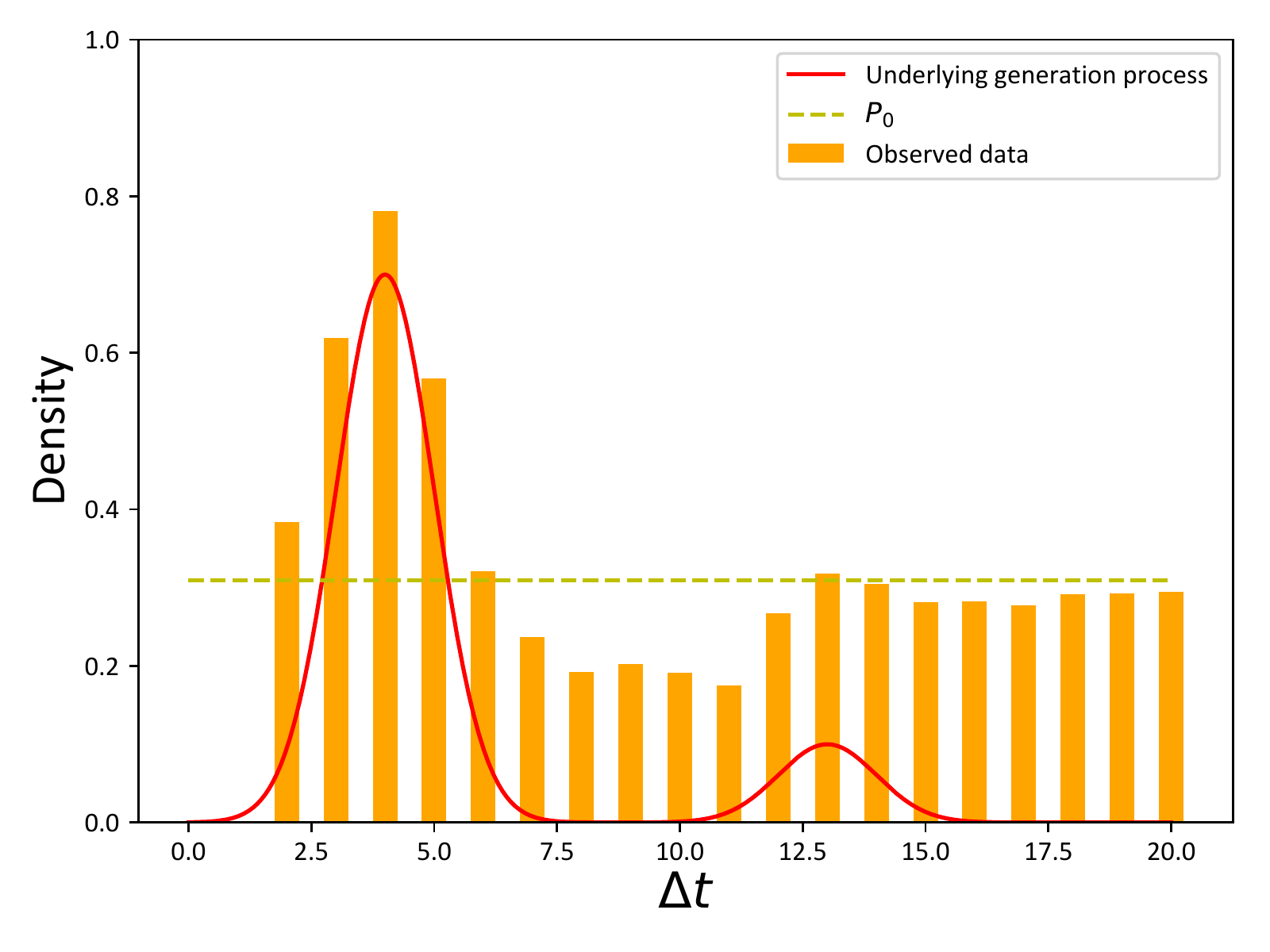}
    \caption{\textbf{Underlying generation process vs observed data} --- The red curve represents the underlying probability of contagion by C given an exposure observed $\Delta t$ steps before C. The orange bars represent the observed probability of such events. We see that there is a noise $P_0(C)$ in the observed data. The underlying generation process can then only be observed in the dataset when its effect is larger than some threshold $P_0(C)$.}
    \label{fig:FigTrueVsObs}
\end{figure}
Because the dataset is built looking at all exposure-contagion correlations in a sequence, there is inherent noise in the resulting data. To illustrate this, we look at the illustrated example Fig.\ref{fig:FigProc} and consider the exposure to C leading to a contagion happening at time $t$. We assume that in the underlying interaction process, the contagion by C at time $t+t_s$ took place only because C appeared at time $t$. However, when building the dataset, the contagion by C is also attributed to A appearing at times $t-\delta t$, $t-3\delta t$ and $t-6\delta t$, and to B appearing at times $t-4\delta t$ and $t-5\delta t$. It induces a noise in the data. In general, for any contagion in the dataset, several observations (pair exposure-contagion) come from the random presence of entities unrelated to this contagion. 

We now illustrate how this problem introduces noise in the data.
In Fig.\ref{fig:FigTrueVsObs}, we see that the actual underlying data generation process (probability of a contagion by C given an exposure present $\Delta t$ step earlier) does not exactly fit the collected resulting data: the data gathering process induces a constant noise whose value is noted $P_0(C)$ --that is the average probability of contagion by C. Thus the interaction effect can only be observed when its associated probability of contagion is larger than $P_0(C)$. Consequently, the performance improvement of a model that accounts for interactions may seem small compared with a baseline that only infers $P_0(C)$. That is we observe in the experimental section. However, in this context, a small improvement in performance shows an extended comprehension of the underlying interacting processes at stake (see Fig.\ref{fig:FigTrueVsObs}, where the red line obviously explains the data better than a constant baseline). Our method efficiently infers $P_0(C)$ \textit{via} a time-independent kernel function $\beta_{i,j}^{P_0(i)}$.

\subsection{Evaluation criteria}
The main difficulty in evaluating these models is that interactions might occur between a small number of entities only. It is the case here, where many pairs of entities have little to no interaction (see the Discussion section). This makes it difficult to evaluate how good a model is at capturing them. To this end, our principal metric is the residual sum of squares (\textbf{RSS}). The RSS is the sum of the squared difference between the observed and the expected frequency of an outcome. This metric is particularly relevant in our case, where interactions may occur between a small number of entities: any deviation from the observed frequency of contagion is accounted for, which is what we aim at predicting here. 
We also consider the Jensen-Shannon (\textbf{JS}) divergence; the JS divergence is a symmetric version of the Kullback–Leibler divergence, which makes it usable as a metric \cite{JSDiv}.

We finally consider the best-case F1-score (\textbf{BCF1}) of the models, that is, the F1-score of the best scenario of evaluation. It is not the standard F1 metric (that poorly distinguishes the models since few interactions occur), although its computation is similar. Explicitly, it generalizes F1-score for comparing probabilities instead of comparing classifications; the closer to 1, the closer the inferred and observed probabilities. It is derived from the best-case confusion matrix, whose building process is as follows: we consider the set of every information that appeared before information i at time $t_i$ in the interval, that we denote $\mathcal{H}_i$. We then compute the contagion probability of i at time $t_i+t_s$ to every exposure event $t_j^{(y)} \in \mathcal{H}_i$. Confronting this probability with the observed frequency f of contagions of i at time $t_i+t_s$ given $t_j^{(y)}$ among N observations, we can build the best-case confusion matrix. 
In the best case scenario, if out of N observations the observed frequency is f and the predicted frequency is p, the number of True Positives is $N \times \min \{ p, f \}$, the number of False Positives is $N \times \min \{ p-f, 0 \}$, the number of True Negatives is $N \times \min \{ 1-p, 1-f \}$, the number of False Positives is $N \times \min \{ f-p, 0 \}$.

Finally, when synthetic data is considered, we also compute the mean squared error of the $\beta$ matrix inferred according to the $\beta$ matrix used to generate the observations, that we note MSE $\beta$.

We purposely ignore evaluation in prediction because, as we show later, interactions influence quickly fades over time: probabilities of contagion at large times are mainly governed by the background noise discussed in previous sections. Therefore, it would be irrelevant to evaluate our approach's predictive power on the whole range of times where it does not bring any improvement over a naive baseline (see Fig.\ref{fig:FigDistrib}). A way to alleviate this problem would be to make predictions only when interactions effects are above/below a certain threshold (at short times, for instance). However, such an evaluation process would be debatable. Here, we choose to focus on the descriptive aspect of InterRate.

\subsection{Baselines}
\subsubsection{Naive baseline}
For a given piece of information i, the contagion probability is defined as the number of times this information is contaminated divided by its number of occurrences.
%For a given pair of entities (i,j), the contagion probability is defined as the fraction of times i leads to a contagion after an exposure to j.
\subsubsection{Clash of the contagions}
We use the work presented in \cite{CoC} as a baseline. In this work, the authors model the probability of a retweet given the presence of a tweet in a user's feed. This model does not look for trends in the way interactions take place (it does not infer an interaction profile), considers discrete time steps (while our model works in a continuous-time framework), and is optimized via a non-convex SGD algorithm (which does not guarantee convergence towards the optimal model). More details on implementation are provided in SI file.

\subsubsection{IMMSBM}
The Interactive Mixed-Membership Stochastic Block Model (IMMSBM) is a model that takes interactions between pieces of information into account to compute the probability of a (non-)contagion \cite{IMMSBM}. Note that this baseline does not take the position of the interacting pieces of information into account (time-independent) and assumes that interactions are symmetric (the effect of A on B is the same as B on A).

\subsubsection{ICIR}
The Independent Cascade InterRate (ICIR) is a reduction of our main IR-RBF model to the case where interactions are not considered. We consider the same dataset, enforcing the constraint that off-diagonal terms of $\beta$ are null. The (non-)contagion of a piece of information i is then determined solely by the previous exposures to i itself.

\section{Results}
\subsection{Synthetic data}

% Pour version ECML-PKDD
\newlength{\lgCase}
\begin{figure}
    \centering
    \includegraphics[width=\textwidth]{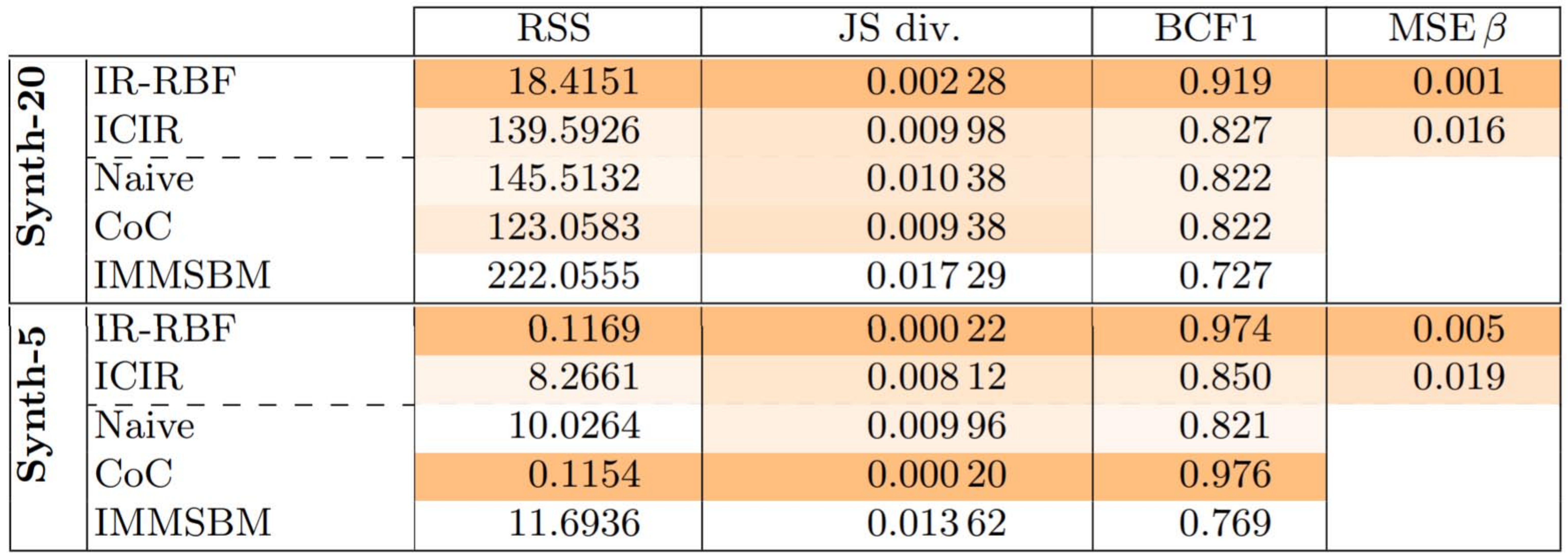}
    \caption{\textbf{Experimental results} --- Darker is better (linear scale). Our model outperforms all of the baselines in almost every dataset for every evaluation metric. The standard deviations of the 5 folds cross-validation are negligible and reported in SI file. \label{tabMetricsSynth}}
\end{figure}

We generate synthetic data according to the process described in Fig.\ref{fig:FigProc} for a given $\beta$ matrix using the RBF kernel family. First, we generate a random matrix $\beta$, whose entries are between 0 and 1. A piece of information is then drawn with uniform probability and can result in a contagion according to $\beta$, the RBF kernel family and its history. We simulate the outcome by drawing a random number and finally increment the clock. The process then keeps on by randomly drawing a new exposure and adding it to the sequence. We set the maximum length of intervals to 50 steps and generate datasets of 20,000 sequences. 

We present in Tab.~\ref{tabMetricsSynth} the results of the various models with generated interactions between 20 (Synth-20) and 5 (Synth-5) entities. The interactions are generated using the RBF kernel, hence the fact we are not evaluating the IR-EXP model --its use would be irrelevant.
The InterRate model outperforms the proposed baselines for every metric considered. It is worth noting that performances of non-interacting and/or non-temporal baselines are good on the JS divergence and F1-score metrics due to the constant background noise $P_0$. For cases where interactions do not play a significant role, IMMSBM and Naive models perform well by fitting only the background noise. By contrast, the RSS metric distinguishes very well the models that are better at modeling interactions.

Note that while the baseline \cite{CoC} yields good results when few interactions are simulated (Synth-5), it performs as bad as the naive baseline when this number increases (Synth-20). This is due to the non-convexity of the proposed model, which struggles to reach a global maximum of the likelihood even after 100 runs (see supplementary materials for implementation details).

\subsection{Real data}

\begin{figure}
    \centering
    \includegraphics[width=\textwidth]{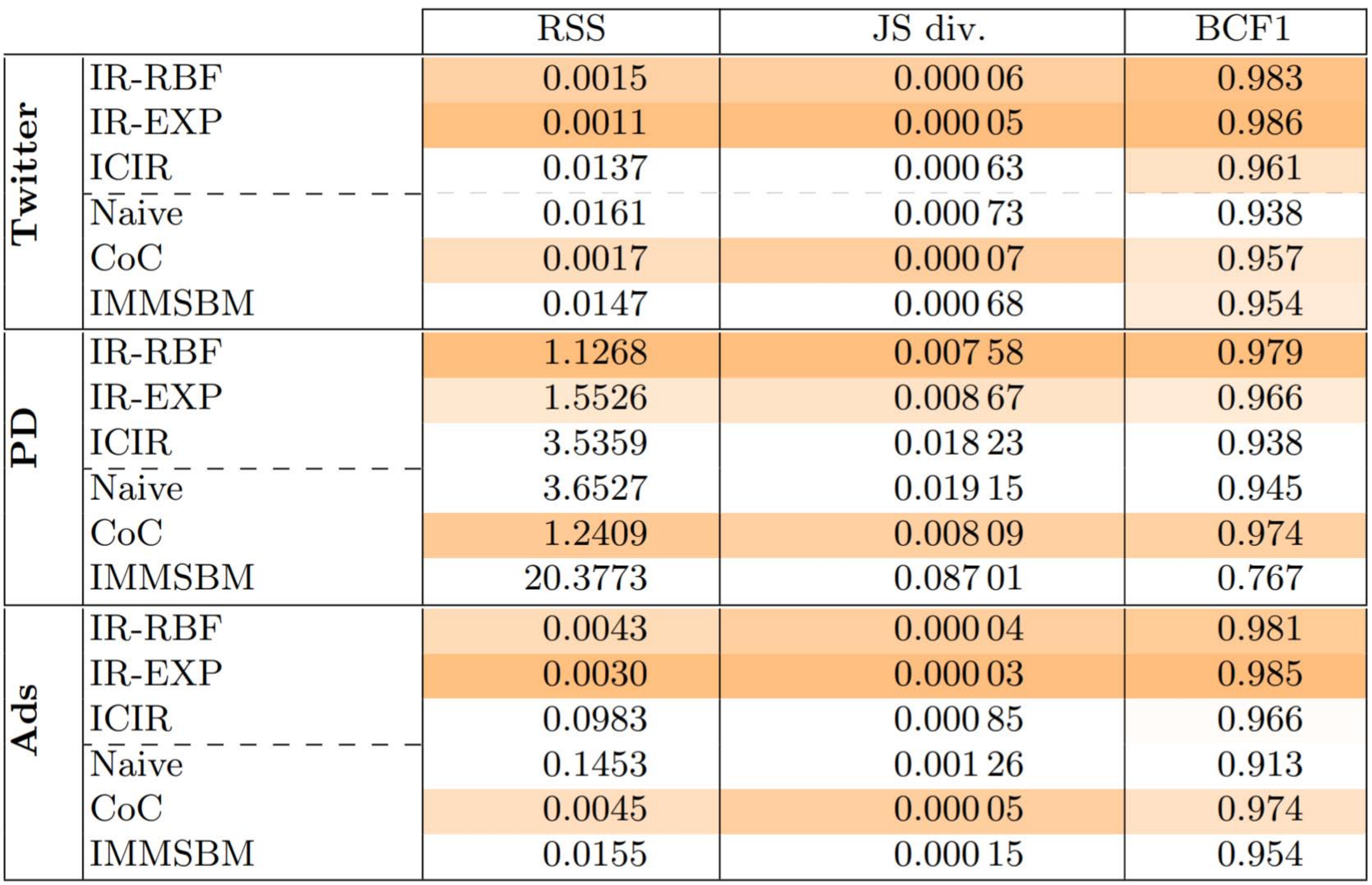}
    \caption{\textbf{Experimental results} --- Darker is better (linear scale). Our model outperforms all of the baselines in almost every dataset for every evaluation metric. The standard deviations of the 5 folds cross-validation are negligible and reported in SI file. \label{tabMetricsRW}}
\end{figure}

We consider 3 real-world datasets. For each dataset, we select a subset of entities that are likely to interact with each other. For instance, it has been shown that the interaction between the various URL shortening services on Twitter is non-trivial \cite{CorrelatedCascade}.
The datasets are a \textbf{Twitter} dataset (104,349 sequences\comment{ of average length 53.5 steps}, exposure are tweets and contagions are retweets) \cite{DataSetTwitter}, a Prisoner's dilemma dataset (\textbf{PD}) (2,337 sequences\comment{ of average length 9.0 steps}, exposures are situations, contagions are players choices) \cite{PrisonerDilemaDataset1,PrisonerDilemaDataset2} and an \textbf{Ads} dataset (87,500 sequences\comment{ of average length 23.9 steps}, exposures are ads, contagions are clicks on ads) \cite{AdsDataset}. A detailed description of the datasets is presented in SI file, section Datasets.

The results on real-world datasets are presented in Tab.\ref{tabMetricsRW}. We see that the IMMSBM baseline performs poorly on the PD dataset: either considering the time plays a consequent role in the probability of contagion, or interactions are not symmetric. Indeed, the core hypothesis of the IMMSBM is that the effect of exposition A on B is the same as B on A, whichever is the time separation between them. In a prisoner's dilemma game setting, for instance, we expect that a player does not react in the same way to defection followed by cooperation as to cooperation followed by defection, a situation for which the IMMSBM does not account. When there are few entities, the CoC baseline performs as good as IR, but fails when this number increases; this is due mainly to the non-convexity of the problem that does not guarantee convergence towards the optimal solution.
Overall, the InterRate models yield the best results on every dataset. 

\section{Discussion}
In Fig.\ref{fig:FigAnalInter}, we represent the interaction intensity over time for every pair of information considered in every corpus fitted with the RBF kernel model. The intensity of the interactions is the inferred probability of contagion minus the base contagion probability in any context: $P_{ij}(t)-P_0(i)$. Therefore, we can determine the characteristic range of interactions, investigate recurrent patterns in interactions, whether the interaction effect is positive or negative, etc. 
\begin{figure}[t!]
    \centering
    \includegraphics[width=1.\columnwidth]{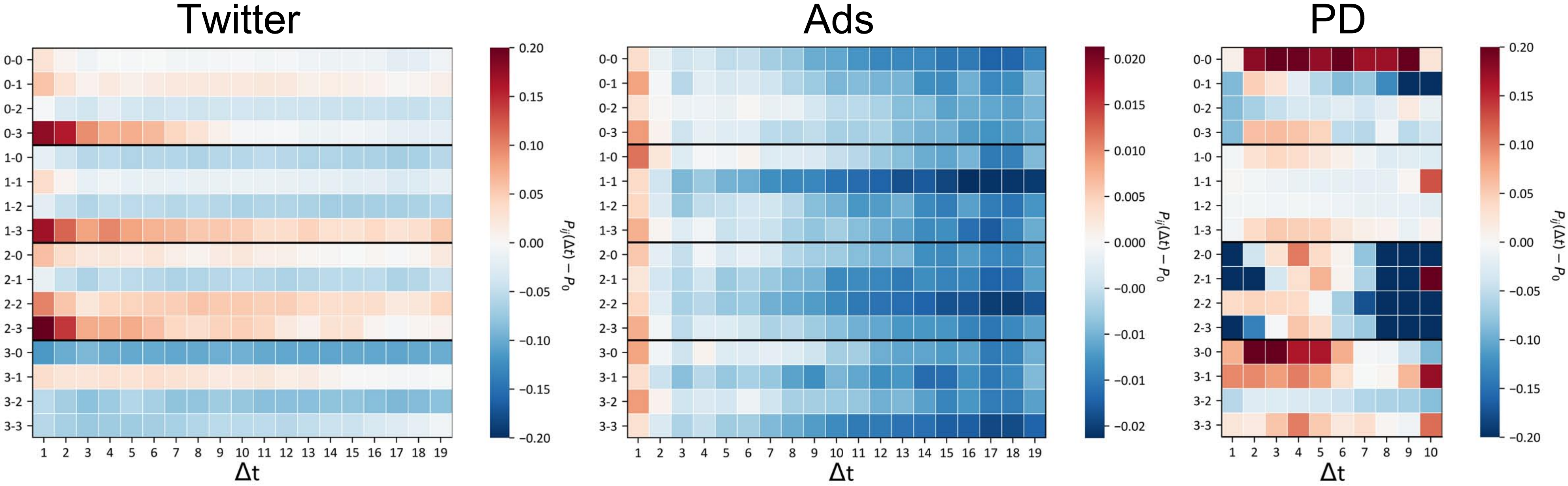}
    \caption{\textbf{Visualization of the interaction profiles} --- Intensity of the interactions between every pair of entities according to their time separation (one line is one pair's interaction profile, similar to Fig.\ref{fig:FigDistrib} seen "from the top"). A positive intensity means that the interaction helps the contagion, while a negative intensity means it blocks it. The key linking numbers on the y-axis to names for each dataset is provided in the Datasets section.}
    \label{fig:FigAnalInter}
\end{figure}
Overall, we understand why the EXP kernel performs as good as the RBF on the Twitter and Ads datasets: interactions tend to have an exponentially decaying influence over time. However, this is not the case on the PD dataset: the effect of a given interaction is very dependent on its position in the history (pike on influence at $\Delta t=3$, shift from positive to negative influence, etc.). 

In the Twitter dataset, the most substantial positive interactions occur before $\Delta$t=3. This finding agrees with previous works, which stated that the most informative interactions within Twitter URL dataset occur within the 3 time steps before the possible retweet \cite{CoC}. We also find that the vast majority of interactions are weak, matching with previous study's findings \cite{CoC,IMMSBM}. However, it seems that tweets still exert influence even a long time after being seen, but with lesser intensity.

In the Prisoner's Dilemma dataset, players' behaviors are heavily influenced by the previous situations they have been exposed to. For instance, in the situation where both players cooperated in the previous round (pairs 2-x, 3$^{rd}$ section in Fig.\ref{fig:FigAnalInter}-PD). The probability that the player defects is then significantly increased if both players cooperated or if one betrayed the other exactly two rounds before but decreased if it has been two rounds that players both cooperate. 

Finally, we find that the interactions play a lesser role in the clicks on ads. We observe a slightly increased probability of click on every ad after direct exposure to another one. We also observe a globally decreasing probability of click when two exposures distant in time, which agrees with previous work's findings \cite{AdsDataset}. Finally, the interaction profile is very similar for every pair of ads; we interpret this as a similarity in users' ads perception.

We showed that for each of the considered corpus, considering the interaction profile provides an extended comprehension of choice adoption mechanisms and retrieves several state-of-the-art conclusions. The proposed graphical visualization also provides an intuitive view of how the interaction occurs between entities and the associated trends, hence supporting its relevance as a new tool for researchers in a broad meaning.

\section{Conclusion}
We introduced an efficient convex model to investigate the way interactions happen within various datasets. Interactions modeling has been little explored in data science, despite recent clues pointing to their importance in modeling real-world processes \cite{CorrelatedCascade,CoC}. Unlike previous models, our method accounts for both the interaction effects and their influence over time (the interaction profile). We showed that this improvement leads to better results on synthetic and various real-world datasets that can be used in different research fields, such as recommender systems, spreading processes, and human choice behavior.
We also discussed the difficulty of observing significant interaction profiles due to the data-gathering process's inherent noise and solved the problem by introducing a time-independent kernel. We finally proposed a way to easily explore the results yielded by our model, allowing one to read the interaction profiles of any couple of entities quickly. 

In future works, we will explore the way interactions vary over time and work on identifying recurrent patterns in interaction profiles. It would be the next step for an extended understanding of the role and nature of interacting processes in real-world applications.

\bibliographystyle{splncs04}
\bibliography{mybibliography}

\end{document}

\title{Supplementary Material for Information Interaction Profile of Choice Adoption}
%\author{Gael Poux-Medard\inst{1}\orcidID{0000-0002-0103-8778} \and
%Julien Velcin\inst{1}\orcidID{0000-0002-2262-045X} \and
%Sabine Loudcher\inst{1}\orcidID{0000-0002-0494-0169}}
%
%\authorrunning{G. Poux-Médard et al.}
%\institute{$^1$ Université de Lyon, ERIC EA 3083, France\\
%\email{gael.poux-medard@univ-lyon2.fr}\\
%\email{julien.velcin@univ-lyon2.fr}\\
%\email{sabine.loudcher@univ-lyon2.fr}\\
%}
%
\maketitle              % typeset the header of the contribution
%

%\comment{
\section{Datasets}
We provide details on the way datasets have been built from raw data. For each of the real-world datasets, we choose to consider only the order of apparition of the various entities instead of their absolute appearance times. This implies setting the time separating two successive exposures as constant, that we note $\delta t$. This choice is supported by state-of-the-art works \cite{CoC} and we observed in our own experiments that it is more relevant than considering absolute times. Besides, we do not consider the first 10 pieces of information of any sequence to avoid boundary effects (the first 5 steps for the PD dataset): the history of exposures is incomplete in this case and could lead to biased results.
For each dataset entities list, the number before the entity name is the key used in Fig.4 of the main article. The entities subsets have been chosen by computing the co-occurrence matrix of all the entities, and then select the ones that are part of a cluster using a K-mean algorithm. The datasets are:

\begin{itemize}
    %\item \textbf{Synthetic data} : First we generate a random matrix $\beta$, whose entries are comprised between 0 and 1. A piece of information is then drawn uniformly at random, and has the possibility to result in a contagion according to $\beta$, the RBF kernel family and its history. We simulate the outcome by drawing a random number, and finally increment the clock. The process then keeps on by randomly drawing a new exposure and adding it to the sequence. We set the maximum length of intervals to 50 steps, and generate datasets of 20,000 sequences. 
    \item \textbf{Twitter} \cite{DataSetTwitter}: it consists in a collection of all the tweets containing URLs that have been posted on Twitter during October 2010, with the associated followers networks. A tweet read by a user in her feed is an exposition, and its possible retweet is a contagion. We consider only the URLs associated with the following URL shortening websites, the same as in \cite{CorrelatedCascade}: \{0: migre.me, 1: bit.ly, 2: tinyurl, 3: t.co\}. The final dataset is made of 104,349 sequences of average length 53.5 steps (1 step = $t_s$), for 1,276,670,965 observed interactions.
    \item Prisoner's dilemma dataset (\textbf{PD}) \cite{PrisonerDilemaDataset1,PrisonerDilemaDataset2}: contains ordered sequences of repeated Prisoner's dilemma game between two players. From the dataset introduced in \cite{PrisonerDilemaDataset1}, we consider the sub-dataset noted BR-risk 0 (first entry of Tab.2 in the reference); we choose this subset in order to have decisions made in an homogeneous context, where players struggle in a dilemma that is hard to solve (which depends on the combination of the parameters T, R, S and P discussed further). Within each round, the players can either defect of cooperate. Each duel is made of 10 rounds. If both cooperate, the reward R is high, if they both defect, the reward P is low, if one player cooperates while the other defects, this one gets a penalty S, while the other gets a reward T. To make the game a Prisoner dilemma, the variables have to obey T$>$R$>$P$>$S.
    We refer to the combination of players' actions ("the user cooperated and the opponent defected at time t") as exposures, and to the \textit{defect} actions of the player in the following round as a contagion. We defined the action of cooperating as a non-contagion. We therefore have 4 possible situations (\{0: Player cooperated and opponent defected, 1: Both players defected, 2: Both players cooperated, 3: Player defected and opponent cooperated\}) and 2 possible outcomes (Player cooperates or defects). The final dataset is made of 2,337 sequences of average length 10.0 steps, for 189,297 observed interactions.
    \item Taobao dataset (\textbf{Ads}): contains all ads exposures for 1,140,000 randomly sampled users from the website of Taobao for 8 days (5/6/2017-5/13/2017) \cite{AdsDataset}. Taobao is one of the largest e-commerce websites, and is owned by Alibaba. To each exposure is associated the corresponding timestamp and the action of the user (click on the ad or not). A click is considered as a contagion. The subset of ads we consider is: \{0: 4520, 1: 4280, 2: 1665, 3: 4282\}. The resulting dataset is made 87,500 sequences of average length 23.9 steps, for 240,932,401 observed interactions.
\end{itemize}
%}

\section{Implementation of Clash of the Contagions}
In this appendix, we provide technical details on the way the Clash of Contagions baseline is implemented. Following the directions given in the reference article \cite{CoC}, we implemented a Stochastic Gradient Descent (SGD) method for parameters inference. Given the small number of entities considered in the experiments, each iteration of the SGD is computed using the full dataset instead of slicing it into mini-batches. 

\subsubsection{Setup}
For each corpus, we run the SGD algorithm 100 times, from which we save the parameters maximizing the likelihood the most. At the beginning of each run, parameters M and $\Delta$ are randomly initialized. The stopping condition makes the algorithm ends when the relative variation of the likelihood according to the last iteration is been lesser than $10^{-6}$ for more than 30 times in a row; those numbers have been chosen empirically to maximize the performances of the algorithm. The hyper-parameters have been set to: T=5 (number of clusters) and K=20 (number of considered time steps).

\subsubsection{Update rule}
In each iteration, we update the parameters in the direction of the gradient descent (noted G). However, a major problem when dealing with SGD is to choose the line step length $\eta$ (the amplitude of the variation of the parameters in the direction of the gradient G). After each iteration, we compare several update rules, and we select the one maximizing the likelihood. Those rules are as follows:
\begin{itemize}
\item AdaGrad: $\eta^{AG} \times G$
\item AdaDelta: $\eta^{AD} \times G$
\item Line search in the direction of the gradient: $\eta^{LS} \times G$
\item Line search in the direction of AdaDelta: $\eta^{LS} \times \eta^{AD} \times G$
\end{itemize}
The line search snippets consider 50 values of $\eta^{LS}$ logarithmically distributed in the interval $[ 10^{-6} ; 10^3 ]$. 

\subsubsection{Constraints on the parameters}
The membership vectors entries $M_{i,t}$ (membership of i to cluster t) must be positive and sum to 1 over all the clusters ($\sum_{t} M_{i,t} = 1$). In order to enforce this constraint, we consider the following variable change: $M_{i,t} \rightarrow \frac{\phi_{i,t}^2}{\sum_{t'} \phi_{i,t'}^2}$. This transformation guarantees the membership vector properties with no need for penalty methods in the implementation.

Besides, as stated in \cite{CoC}, it can happen that a probability is larger then 1 or lesser than 0. In the absence of complementary details in the main paper, we implemented our own method to force the probabilities to stay within reasonable bounds. Here it is impossible to make a simple variable change to enforce this constraint, since the probability results of a non-linear combination of the model's parameters. We added to the likelihood an exponential penalty term. Let P denote a quantity we want to constrain between 0 and 1. The penalty term equals $e^{-\lambda P} + e^{\lambda(P-1)}$. $\lambda$ here is a parameter that tunes the intensity of the penalty, and is empirically set to $\lambda=75$. This penalty function has the form of a well with very steep walls in x=0 and x=1. In this way, it seldom happens that probabilities are larger than 1 or lesser than 0, as said in the main article. When such cases happen, we simply set it back to the closest bound for methods comparisons.

\clearpage
\section{Experimental results with standard deviation}
The numerical results along with the associated standard deviation are presented Table \ref{tabMetrics}.
\begin{table*}[h!]
\centering

\caption{\textbf{Experimental results with associated standard deviation} \label{tabMetrics}}

\newlength{\lgCase}
\setlength{\lgCase}{2.8cm}
\begin{tabular}{p{0.21\lgCase}p{0.5\lgCase}|p{0.67\lgCase}|p{1.2\lgCase}|p{0.7\lgCase}|p{0.7\lgCase}|p{0\lgCase}}
\cline{3-6}
   & & \,\,\,\,\,\,\,\,\,\,\,\,RSS & \,\,\,\,\,\,\,\,\,\,\,\,\,\,\,\,\,\,\,\,\,\,\,\,JS div. & \,\,\,\,\,\,\,\,\,\,\,BCF1 & \,\,\,\,\,\,\,\,\,\,\,MSE\,$\beta$ & \\
 \cline{1-6}
\end{tabular}

\begin{tabular}{|p{0.2\lgCase}|p{0.5\lgCase}|S[table-column-width=0.67\lgCase,round-precision=4]|S[table-column-width=1.2\lgCase,round-precision=5]|S[table-column-width=0.7\lgCase,round-precision=3]|S[table-column-width=0.7\lgCase,round-precision=3]|p{0\lgCase}}
    \multirow{4}{*}{\rotatebox[origin=c]{90}{\textbf{Synth-20}}} &  IR-RBF &   \SI{18.42(25)}{} &  \SI{0.002284(45)}{} &  \SI{0.9188(1)}{} & \SI{0.00503(4)}{} & \\
    
     &  ICIR &  \SI{139.59(77)}{} &  \SI{0.009980(49)}{} & \SI{0.8271(16)}{} &   \SI{0.01588(0)}{} & \\
     
     \cdashline{2-6}
     
     &  Naive & \SI{145.51(75)}{} & \SI{0.010379(59)}{} & \SI{0.8221(2)}{} &  & \\
     
     &  CoC &  \SI{123.06(75)}{} &  \SI{0.009384(70)}{} &  \SI{0.8220(15)}{} &  & \\
     
     &  IMMSBM &  \SI{222.06(139)}{} &  \SI{0.017288(51)}{} & \SI{0.7265(7)}{} &  & \\
     
 \cline{1-6}
\end{tabular}
\begin{tabular}{|p{0.2\lgCase}|p{0.5\lgCase}|S[table-column-width=0.67\lgCase,round-precision=4]|S[table-column-width=1.2\lgCase,round-precision=5]|S[table-column-width=0.7\lgCase,round-precision=3]|S[table-column-width=0.7\lgCase,round-precision=3]|p{0\lgCase}}
    \multirow{4}{*}{\rotatebox[origin=c]{90}{\textbf{Synth-5}}} &  IR-RBF &  \SI{0.117(4)}{} &  \SI{0.000217(6)}{} & \SI{0.9742(4)}{} &   \SI{0.00530(3)}{} & \\
    
     &  ICIR &  \SI{8.266(36)}{} &  \SI{0.008117(38)}{} & \SI{0.8499(11)}{} & \SI{0.01921(1)}{} &  \\
     
     \cdashline{2-6}
     
     &  Naive & \SI{10.026(48)}{} &  \SI{0.009956(48)}{} & \SI{0.8214(17)}{} &  & \\
     
     &  CoC & \SI{0.115(23)}{} &  \SI{0.000197(31)}{} & \SI{0.9763(10)}{} &  & \\
     
     &  IMMSBM & \SI{11.694(274)}{} &  \SI{0.013622(489)}{} & \SI{0.7693(50)}{} &  & \\
     
 \cline{1-6}
\end{tabular}
\begin{tabular}{|p{0.2\lgCase}|p{0.5\lgCase}|S[table-column-width=0.67\lgCase,round-precision=4]|S[table-column-width=1.2\lgCase,round-precision=5]|S[table-column-width=0.7\lgCase,round-precision=3]|p{0.7\lgCase}p{0\lgCase}}
    \multirow{4}{*}{\rotatebox[origin=c]{90}{\textbf{Twitter\,\,\,\,\,\,}}} &  IR-RBF &  \SI{0.0015(4)}{} &  \SI{0.000058(8)}{} & \SI{0.9832(8 )}{} & & \\
     
     &  IR-EXP &  \SI{0.0011(2)}{} &  \SI{0.000049(5)}{} & \SI{0.9862(6)}{} &  & \\
     
     &  ICIR &  \SI{0.0137(8)}{} &  \SI{0.000629(48)}{} & \SI{0.9614(12)}{} &  & \\
     
     \cdashline{2-5}
     
     &  Naive &  \SI{0.0161(9)}{} &  \SI{0.000725(60)}{} & \SI{0.9379(38)}{} &  & \\
     
     &  CoC &  \SI{0.0017(2)}{} &  \SI{0.000067(4)}{} & \SI{0.9572(263)}{} &  & \\
     
     &  IMMSBM &  \SI{0.0147(13)}{} &  \SI{0.000683(73)}{} & \SI{0.9543(29)}{} &  & \\
     
 \cline{1-5}
\end{tabular}

\begin{tabular}{|p{0.2\lgCase}|p{0.5\lgCase}|S[table-column-width=0.67\lgCase,round-precision=4]|S[table-column-width=1.2\lgCase,round-precision=5]|S[table-column-width=0.7\lgCase,round-precision=3]|p{0.7\lgCase}p{0\lgCase}}
     \multirow{4}{*}{\rotatebox[origin=c]{90}{\textbf{PD\,\,\,\,}}} &  IR-RBF &  \SI{1.13(18)}{} &  \SI{0.007583(470)}{} & \SI{0.9789(27)}{} &  & \\
     
     &  IR-EXP &  \SI{1.55(31)}{} &  \SI{0.008669(1137)}{} & \SI{0.9661(22)}{} &  & \\
     
     &  ICIR &  \SI{3.54(31)}{} &  \SI{0.018225(1286)}{} & \SI{0.9381(47)}{} &  & \\
     
     \cdashline{2-5}
     
     &  Naive &  \SI{3.65(43)}{} &  \SI{0.019147(1538)}{} & \SI{0.9455(23)}{} &  & \\
     
     &  CoC &  \SI{1.24(28)}{} &  \SI{0.008088(1266)}{} & \SI{0.9736(56)}{} &  & \\
     
     &  IMMSBM &  \SI{20.38(297)}{} &  \SI{0.087010(12172)}{} & \SI{0.7672(204)}{} &  & \\
     
 \cline{1-5}
\end{tabular}

\begin{tabular}{|p{0.2\lgCase}|p{0.5\lgCase}|S[table-column-width=0.67\lgCase,round-precision=4]|S[table-column-width=1.2\lgCase,round-precision=5]|S[table-column-width=0.7\lgCase,round-precision=3]|p{0.7\lgCase}p{0\lgCase}}

      \multirow{4}{*}{\rotatebox[origin=c]{90}{\textbf{Ads\,\,\,\,\,\,}}} &  IR-RBF &   \SI{0.0043(6)}{} &  \SI{0.000043(6)}{} & \SI{0.9814(19)}{} & & \\
     
     &  IR-EXP &  \SI{0.0030(3)}{} &  \SI{0.000030(2)}{} & \SI{0.9852(7)}{} &  & \\
     
     &  ICIR &  \SI{0.0983(70)}{} &  \SI{0.000848(63)}{} & \SI{0.9659(11)}{} &  & \\
     
     \cdashline{2-5}
     
     &  Naive &   \SI{0.1453(106)}{} &  \SI{0.001263(95)}{} & \SI{0.9126(59)}{} & & \\
     
     &  CoC &  \SI{0.0045(2)}{} &  \SI{0.000045(2)}{} & \SI{0.9741(40)}{} &  & \\
     
     &  IMMSBM &   \SI{0.0155(18)}{} &  \SI{0.000153(15)}{} & \SI{0.9543(16)}{} & & \\
     
 \cline{1-5}
\end{tabular}
\end{table*}
%}

\bibliographystyle{splncs04}
\bibliography{mybibliography}